\documentclass{article}
\usepackage{ijcai21}

\usepackage{times}
\usepackage{soul}
\usepackage{url}
\usepackage[hidelinks]{hyperref}
\usepackage[utf8]{inputenc}
\usepackage[small]{caption}
\usepackage{graphicx}
\usepackage{amsmath}
\usepackage{amsthm}
\usepackage{booktabs}
\usepackage{algorithm}
\usepackage{algorithmic}
\usepackage{listings}
\usepackage{subcaption}
\usepackage[colorinlistoftodos,prependcaption,textsize=tiny]{todonotes}

\newtheorem{example}{Example}
\newtheorem{theorem}{Theorem}
\newtheorem{definition}{Definition}
\newtheorem{corollary}{Corollary}

\title{grASP: A Graph Based ASP-Solver and Justification System}

\author{
Fang Li$^1$\footnote{Contact Author}\and
Huaduo Wang$^1$\and
Gopal Gupta$^1$\\
\affiliations
$^1$University of Texas at Dallas\\
\emails
\{fang.li, huaduo.wang, gupta\}@utdallas.edu
}

\begin{document}

\maketitle

\begin{abstract}

Answer set programming (ASP) is a popular nonmonotonic-logic based paradigm for  knowledge representation and solving combinatorial problems. Computing the answer set of an ASP program is NP-hard in general, and researchers have been investing significant effort to speed it up. The majority of current ASP solvers employ SAT solver-like technology to find these answer sets. As a result, justification for why a literal is in the answer set is hard to produce. There are  dependency graph based approaches to find answer sets, but due to the representational limitations of dependency graphs, such approaches are limited. We propose a novel dependency graph-based approach for finding answer sets in which conjunction of goals is explicitly represented as a node which allows arbitrary answer set programs to be uniformly represented. Our representation preserves causal relationships allowing for justification for each literal in the answer set to be elegantly found. Performance results from an implementation are also reported. Our work paves the way for computing answer sets without grounding a program.
  
\end{abstract}

\section{Introduction} \label{sec:introduction}

Answer set programming (ASP) \cite{GL2,MT5,EFPL12,SNS18} is a popular nonmonotonic-logic based paradigm for  knowledge representation and solving combinatorial problems. Computing the answer set of an ASP program is NP-hard in general, and researchers have been investing significant effort to speed it up. Most ASP solvers employ SAT solver-like technology to find these answer sets. As a result, justification for why a literal is in the answer set is hard to produce. There are dependency graph (DG) based approaches to find answer sets, but due to the representational limitations of dependency graphs, such approaches are limited. In this paper we propose a novel dependency graph-based approach for finding answer sets in which conjunctions of goals is explicitly represented as a node which allows arbitrary answer set programs to be uniformly represented. Our representation preserves causal relationships allowing for justification for each literal in the answer set to be elegantly found. Performance results from an implementation are also reported. Our work paves the way for computing answer sets without grounding a program.

Compared to SAT solver based implementations,  graph-based implementations of ASP have not been well studied. Very few researchers have investigated graph-based techniques. NoMoRe system \cite{anger2001nomore} represents ASP programs with a \textit{block graph} (a labeled graph) with meta-information, then computes the A-coloring (non-standard graph coloring with two colors) of that graph to obtain answer sets. Another approach \cite{konczak2005graphs} uses \textit{rule dependency graph} (nodes for rules, edges for rule dependencies) to represent ASP programs, then performs graph coloring algorithm to determine which rule should be chosen to generate answer sets. Another group \cite{linke2005suitable} proposed an hybrid approach which combines different kinds of graph representations that are suitable for ASP. The hybrid graph uses both rules and literals as nodes, while edges represent dependencies. It also uses the A-coloring technique to find answer sets.

All of the above approaches were well designed, but their graph representations are complex as they all rely on extra information to map the ASP elements to nodes and edges of a graph. In contrast, our approach uses a much simpler graph representation, where nodes represent literals and an edge represent the relationship between the nodes it connects. Since this representation faithfully reflects the casual relationships, it is capable of producing causal justification for goals entailed by the program.


\section{Background} \label{sec:background}

\subsection{Answer Set Programming} \label{sec:asp}   

Answer Set Programming (ASP) is a declarative paradigm that extends logic programming with negation-as-failure. ASP is a highly expressive paradigm that can elegantly express complex reasoning methods, including those used by humans, such as default reasoning, deductive and abductive reasoning, counterfactual reasoning, constraint satisfaction~\cite{baral,gelfond2014knowledge}.

ASP supports better semantics for negation ({\it negation as failure}) than does standard logic programming and Prolog. An ASP program consists of rules that look like Prolog rules. The semantics of an ASP program {$\Pi$} is given in terms of the answer sets of the program \texttt{ground($\Pi$)}, where \texttt{ground($\Pi$)} is the program obtained from the substitution of elements of the \textit{Herbrand universe} for variables in $\Pi$~\cite{baral}.
Rules in an ASP program are of the form:
\begin{equation}
\resizebox{.91\linewidth}{!}{$
    \texttt{p :- q$_1$, ..., q$_m$, not r$_1$, ..., not r$_n$.}
    $}
\label{rule1}
\end{equation}
\noindent where $m \geq 0$ and $n \geq 0$. Each of \texttt{p} and \texttt{q$_i$} ($\forall i \leq m$) is a literal, and each \texttt{not r$_j$} ($\forall j \leq n$) is a \textit{naf-literal} (\texttt{not} is a logical connective called \textit{negation-as-failure} or \textit{default negation}). The literal \texttt{not r$_j$} is true if proof of {\tt r$_j$} \textit{fails}. Negation as failure allows us to take actions that are predicated on failure of a proof. 
Thus, the rule {\tt r :- not s.} states that {\tt r} can be inferred if we fail to prove {\tt s}. 
Note that in Rule \ref{rule1}, {\tt p} is optional. Such a headless rule is called a constraint, which states that conjunction of {\tt q$_i$}'s and \texttt{not r$_j$}'s should yield \textit{false}. Thus, the constraint {\tt :- u, v.} states that {\tt u} and {\tt v} cannot be both true simultaneously in any model of the program (called an answer set).

The declarative semantics of an Answer Set Program \texttt{P} is given via the Gelfond-Lifschitz transform~\cite{baral,gelfond2014knowledge} in terms of the answer sets of the program \texttt{ground($\Pi$)}. 
More details on ASP can be found elsewhere~\cite{baral,gelfond2014knowledge}. 


\subsection{Dependency Graph}\label{sec:dg}

A dependency graph \cite{linke2005suitable} uses nodes and directed edges to represent dependency relationships of an ASP rule. 

\begin{definition}
The dependency graph of a program is defined on its literals s.t. there is a positive (resp. negative) edge from $p$ to $q$ if $p$ appears positively (resp. negatively) in the body of a rule with head $q$.
\label{def1}
\end{definition}

Conventional dependency graphs are not able to represent ASP programs uniquely. This is due to the inability of dependency graphs to distinguish between non-determinism (multiple rules defining a proposition) and conjunctions (multiple conjunctive sub-goals in the body of a rule) in logic programs. For example, the following two programs have identical dependency graphs (Figure \ref{fig:fig1}).

\begin{lstlisting}[language=prolog]
    %% program 1
    p :- q, not r, not p.
    %% program 2
    p :- q, not p.       p :- not r.
\end{lstlisting}

\begin{figure}[tb]
    \centering
    \includegraphics[scale=0.3]{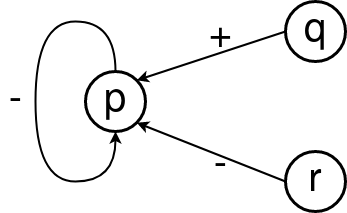}
    \caption{Dependency Graph for Programs 1 \& 2}
    \label{fig:fig1}
\end{figure}

To make conjunctive relationships representable by dependency graphs, we first transform it slightly to come up with a novel representation method. This new representation method, called conjunction node representation (CNR) graph, uses an artificial node to represent conjunction of sub-goals in the body of a rule. This conjunctive node has a directed edge that points to the rule head (Figure \ref{fig:fig2}).

\begin{figure}[tb]
    \centering
    \begin{subfigure}[b]{0.5\linewidth}
        \centering
        \includegraphics[scale=0.3]{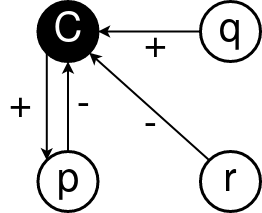}
        \caption{CNR for Program 1}
        \label{fig:fig2a}
    \end{subfigure}\hfill
    \begin{subfigure}[b]{0.5\linewidth}
        \centering
        \includegraphics[scale=0.3]{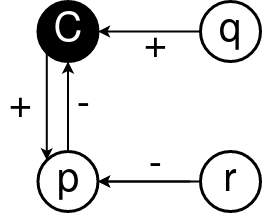}
        \caption{CNR for Program 2}
        \label{fig:fig2b}
    \end{subfigure}
    \caption{CNRs for program 1 \& 2}
    \label{fig:fig2}
\end{figure}

The conjunction node, which is colored black, refers to the conjunctive relation between the in-coming edges from nodes representing subogals in the body of a rule. Note that a CNR graph is not a conventional dependency graph.

\subsubsection{Converting CNR Graph to Dependency Graph} \label{sec:cnrtodg}
Since CNR graph does not follow the dependency graph convention, we need to convert it to a proper dependency graph in order to perform dependency graph-based reasoning. We use a simple technique to convert a CNR graph to an equivalent conventional dependency graph. We negate all in-edges and out-edges of the conjunction node. This process essentially converts a conjunction into a disjunction. Once we do that we can treat the conjunction node as a normal node in a dependency graph. As an example, Figure \ref{fig:fig3} shows the CNR graph to dependency graph transformation for Program 3:
\begin{figure}[tb]
    \centering
    \includegraphics[scale=0.3]{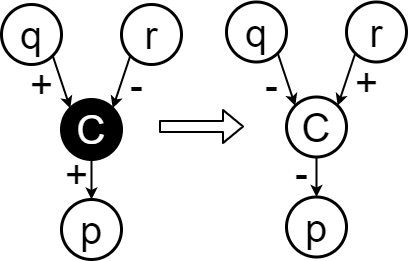}
    \caption{CNR-DG Transformation}
    \label{fig:fig3}
\end{figure}

\begin{lstlisting}[language=prolog]
    %% program 3
    p :- q, not r.
\end{lstlisting}

This transformation is a simple application of De Morgan's law. The rule in program 3 represents:

\begin{lstlisting}[language=prolog]
    p :- C.      C :- q, not r.
\end{lstlisting}

\noindent The transformation produces the equivalent rules:

\begin{lstlisting}[language=prolog]
   p :- not C.  C :- not q.  C :- r.
\end{lstlisting}

\noindent Since conjunction nodes are just helper nodes which allow us to perform dependency graph reasoning, we don't report them in the final answer set.

\subsubsection{Constraint Representation} \label{sec:constraintnode}
ASP also allows for special types of rules called constraints. There are two ways to encode constraints: (i) headed constraint where negated head is called directly or indirectly in the body (e.g., Program 4), and (ii) headless constraints (e.g., Program 5). 

\begin{lstlisting}[language=prolog]
    %% program 4
    p :- not q, not r, not p.
    
    %% program 5
    :- not q, not r.
\end{lstlisting}

Our algorithm models these constraint types separately. For the former one, we just need to apply the CNR-DG transformation directly. Note that the head node connects to the conjunction node both with an in-coming edge and an out-going edge (Figure \ref{fig:fig4a}). For the headless constraint, we create a head node with truth value as \textit{False}. 

\begin{figure}[tb]
    \centering
    \begin{subfigure}[b]{0.5\linewidth}
        \centering
        \includegraphics[scale=0.3]{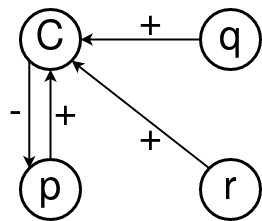}
        \caption{Program 4}
        \label{fig:fig4a}
    \end{subfigure}\hfill
    \begin{subfigure}[b]{0.5\linewidth}
        \centering
        \includegraphics[scale=0.3]{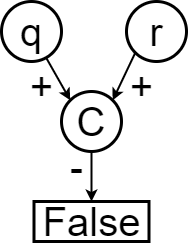}
        \caption{Program 5}
        \label{fig:fig4b}
    \end{subfigure}
    \caption{Constraint DG}
    \label{fig:fig4}
\end{figure}

The reason why we don't treat a headless constraint the same way as a headed constraint is because in the latter case, if head node ({\tt p} in Program 4) is provable through another rule, then the headed constraint is inapplicable. Therefore, we cannot simply assign a false value to its head.

\subsection{Cycles in Program} \label{sec:cycle}

Answer set programs have a notion of odd and even loops over negation as well as positive loops. We refer to loops as cycles here.  Note that a positive (resp. negative) edge is an edge labeled with `+' (resp. `-').

\begin{definition}
A cycle $C$ in a dependency graph is a Positive Cycle iff all edges in $C$ are positive.
\label{def_pc}
\end{definition}

\begin{definition}
A cycle $C$ in a dependency graph is a Negative Even Cycle (NEC) iff $C$ has even number of negative edges.
\label{def_nec}
\end{definition}

\begin{definition}
A cycle $C$ is a Negative Odd Cycle (NOC) iff $C$ has odd number of negative edges.
\label{def_noc}
\end{definition}

\noindent 
Based on above definitions and the stable model semantics, three results trivially follow:

\begin{theorem} 
If $C$ is a positive cycle, then all nodes in $C$ will have the same truth value. 
\label{theo_pc}
\end{theorem}

\begin{theorem}
If $C$ is a NEC, then it will admit multiple answer sets obtained by breaking arbitrary negative edges. 
\label{theo_nec}
\end{theorem}

\noindent By \textit{breaking a cycle}, we mean assigning appropriate truth values to the nodes in the cycles obeying edge dependencies, and then deleting the edges. For example, if we have the graph corresponding to:
\begin{lstlisting}[language=prolog]
    p :- not q.      q :- not p.
\end{lstlisting}
\noindent then breaking this negative even cycle will produce two answer sets: \texttt{$\{$p$\}$} and \texttt{$\{$q$\}$}: one in which node {\tt p} is assigned true and {\tt q} false, and the other the opposite.

\begin{theorem}
If $C$ is a negative odd cycle, then the program $\pi$ is satisfiable iff there is at least one node assigned true in $C$.
\label{theo_noc}
\end{theorem}

\noindent If we have the graph corresponding to:
\begin{lstlisting}[language=prolog]
p :- not q.   q :- not r.   r :- not p.
\end{lstlisting}
\noindent then there will be no models unless one or more of {\tt p}, {\tt q}, or {\tt r} are true through other rules and are assigned the value \textit{True}.

\section{A Graph Algorithm for  Answer Sets} \label{sec:algorithm}

We have developed the \textbf{grASP} graph-based algorithm for finding answer sets. The philosophy of \textbf{grASP} is to translate an ASP program into a dependency graph, then propagate truth values from nodes whose values are known to other connected nodes, obeying the sign on the edge, until the values of all the nodes are fixed. However, due to possible existence of a large number of cycles, the propagation process is not straightforward. In \textbf{grASP}, we define a collection of rules for propagating values among nodes involved in cycles. These assignment rules take non-monotonicity of answer set program and the causal relationship among nodes in the dependency graph into account.

Unlike other SAT-solver based approaches, our graph based approach enables stratification of ASP programs based on dependence. The Splitting Theorem \cite{lifschitz1994splitting} can thus be  used to link the various levels, permitting values to be propagated among nodes more efficiently. Also, the existence of sub-structures (sub-graphs) makes an efficient recursive implementation algorithm possible.

\subsection{Input} \label{sec:input}

At present, the \textbf{grASP} algorithm takes only pure grounded propositional ASP programs as input. A valid rule should be in the form of $head$ :- $body.$, :- $body.$, or $head.$ For example, if we want to represent 3 balls, the form $ball(1..3)$ is invalid. Instead, we have to declare them separately as $ball(1).$ $ball(2).$ $ball(3).$
The input ASP program will be converted and saved into a directed-graph data structure. The conversion process is based on the concepts that were introduced in Section \ref{sec:dg}. 

\subsection{The grASP Algorithm} \label{sec:graspalgorithm}

The  \textbf{grASP} algorithm is designed in a recursive nature. Since a dependency graph represents the causal relationships among nodes, the reasoning should follow a topological order. We don't need to do topological sorting to obtain the order, instead, for each iteration, we just pick those nodes which have no in-coming edges. We call this kind of node a \textbf{root node}. After picking the root nodes, the algorithm checks their values. If a root node's value has not been fixed (no value yet), we assign \textit{False} to it. Otherwise, the root node will keep its value as is. Once all root nodes' values are fixed, we will propagate the values along their out-going edges in accordance with the sign on each edge (the propagation rules will be discussed in Section \ref{sec:propagate}). At the end of this iteration, we remove all root nodes from the graph, then pass the rest of the graph to the recursive call for the next iteration.

The input graph may contain cycles, and, of course, there will be no root nodes in a cycle. Therefore, this recursive process will leave a cycle unchanged. To cope with this issue, we proposed a novel solution, which wraps all nodes in the same cycles together, and treat the wrapped nodes as a single \textit{virtual} node. All the in-coming and out-going edges connecting the wrapped nodes to other nodes will be incident on or emanate from the virtual node. Thus, the graph is rendered acyclic and ready for the root-finding procedure. 

For each iteration of the recursive procedure, we have to treat regular root nodes and virtual root nodes differently. If the node is a regular node, we do the value assignment, but if it is a virtual node, we will have to \textit{break the cycles}. Cycle breaking means that we will remove the appropriate cycle edges by assigning truth values to the nodes involved (\textit{cycle breaking} will be discussed in this section later). After cycle breaking, we will pass the nodes and edges in this virtual node to another recursive call, because the virtual node can be seen as a substructure of the program. The returned value of the recursive call will be the answer set of the program constituting the virtual node. When all regular and virtual root nodes are processed, we will have to merge the values for propagation.

The value propagation in each iteration makes use of the \textbf{splitting theorem} \cite{lifschitz1994splitting} (details omitted due to lack of space). After removing root nodes,  rest of the graph acts as the \textit{top} strata and all of the predecessors constitute the \textit{bottom} strata, using the terminology of \cite{lifschitz1994splitting}. Thus, when we reach the last node in the topological order, we will get the whole answer set.

The cycle breaking procedure may return multiple results, because a negative even cycle generates two worlds (as discussed in Section \ref{sec:cycle}). Therefore, the merging of solution for the root nodes may possibly result in exponential number of solutions. For example, if the root nodes consists of one regular node and two virtual nodes, each virtual node generates two worlds \& the merging process will return four worlds. Of course, this exponential behavior is inherent to ASP.

\begin{algorithm}[tb]
{\small 
	\caption{grASP core algorithm}
	\begin{algorithmic}[1]
	\STATE \textbf{reasoningRec}(Graph g)
	    \IF{$g$ is $empty$}
	        \RETURN $g.nodes.values$
	    \ENDIF
	    \STATE $roots \gets findRoot(g)$
	    \STATE $regular \gets new$ $List()$
	    \STATE $virtual \gets new$ $List()$
        \FOR{$node$ in $roots$}
    	    \IF{$node.type == regular$}
    	        \IF{$node.value == None$}
    	           \STATE $node.value \gets False$
    	        \ENDIF
    	        \STATE $regular.append(node.value)$
    	   \ELSE
    	        \STATE $worlds \gets breakCycle(node)$
    	        \STATE $current \gets new$ $List()$
    	        \FOR{$w$ in $worlds$}
    	            \STATE $current.append($\textbf{reasoningRec}$(w))$
    	        \ENDFOR
    	        \STATE $virtual.append(values)$
    	   \ENDIF
    	\ENDFOR
    	
    	\STATE $root\_values \gets merge\_roots(regular, virtual)$
        \STATE $results \gets new$ $List()$
        \FOR{$sub\_world$ in $root\_values$}
            \STATE $sub\_g \gets g.copy()$
    	    \STATE $propagate(sub\_g, sub\_world)$
    	    \STATE $sub\_g.remove(roots)$
    	    \FOR{$ans$ in \textbf{reasoningRec}$(sub\_g)$}
    	        \STATE $results.append(sub\_world.extend(ans))$
    	    \ENDFOR
    	\ENDFOR
    	\RETURN $results$
	\end{algorithmic}
	\label{alg}}
\end{algorithm}

Algorithm \ref{alg} is the pseudocode of \textbf{grASP}'s core algorithm. Line 9-13 assigns values to regular root nodes, while line 14-20 deals with value assignments for virtual nodes. Cycle breaking happens at line 15. Line 23 merges all worlds generated by the root nodes, and prepares these values as the \textit{bottom} strata (Splitting Theorem \cite{lifschitz1994splitting}) for the next layer. Lines 24-32 deals with the recursive calls on the rest of the graph.

Note that we can optimize the system by adding consistency checking into the regular root node processing (line 11). Since the constraint node is always \textit{False} (see Section \ref{sec:constraintnode}), the algorithm can narrow the search by tracing back along its in-coming edges. For example, if a negative edge is incident into the constraint node, we will know that the node on the other end cannot be \textit{False}. Then if that node happens to be assigned a \textit{False} value,  we stop the search along that path.
 
\subsubsection{Propagation Rules} \label{sec:propagate}

In an ASP rule, the head term only can be assigned as \textit{True} if all its body term(s) are true. For example, in rule {\tt p :- not q.}, only when {\tt q} is unknown or known as \textit{False}, {\tt p} will be \textit{True}. For another example {\tt p :- q.}, {\tt p} will be \textit{True}, only when {\tt q} is known as \textit{True}. In both examples, {\tt p} will not be assigned as \textit{False}, until the reasoning of the whole program fails to make it \textit{True}. Therefore, mapping this to our graph representation, we obtain two propagation rules: (i) when a node $N$ has a \textit{True} value, assign \textit{True} to all the nodes connected to $N$ via positive out-going edges of $N$; (ii) when a $N$ node has a \textit{False} value, assign \textit{True} to all the nodes connected to $N$ via negative out-going edges of $N$.

\subsubsection{\bf Cycle Wrapping}
As previously mentioned (Section \ref{sec:graspalgorithm}), those nodes which are involved in the same cycles need to be wrapped into a virtual node. The reason being that we want the tangled nodes to act like a single node, in order to be found as a root. This requires the dependency of the wrapped nodes to be properly handled. The virtual node should inherit the dependencies of all the node it contains. These dependency relations include both incoming and outgoing edges.

Since cycles may be overlapped or nested with each other, we can make use of the strongly connected component concept in graph theory. Thus, each strongly connected component will be a virtual node.

\subsubsection{Cycle Breaking} \label{sec:cyclebreaking}
We can state the following corollary:

\begin{corollary}
In the dependency graph of an answer set program, if a node's value is True, all of its in-coming edges and negative out-going edges can be removed. If a node's value is False, then all its positive out-edges can be removed.
\label{corol}
\end{corollary}

\begin{proof}
According to the propagation rules (discussed in Section \ref{sec:propagate}), a node can only be assigned \textit{True} through in-coming edges. When a node has already been known as \textit{True}, it no longer needs any assignment, then all in-coming edges become meaningless. Also, a negative edge won't be able to propagate the \textit{True} value to the other side. If a node has been known as \textit{False}, we still have to keep its in-coming edges to detect inconsistency (if some of its predecessors attempt to assign it as \textit{True}, the program is inconsistent). Note that a node labeled \textit{False} cannot make any node \textit{True} through an outgoing positive edge.
\end{proof}
\noindent We will use Corollary \ref{corol} to remove edges while breaking cycles. 
%
%
%
According to the algorithm design, cycles will only exit in virtual nodes, and a virtual node will only be visited as a root. Therefore, when we start looking at a cycle that needs to be broken, it means that there is no in-coming edge from outside connected to any node inside the cycle. As we know that there are only three types of cycles: positive cycles (PC), negative even cycles (NEC), negative odd cycles (NOC) (see Section \ref{sec:cycle}), the tangled cycles in a virtual node can be divided into four cases. We will illustrate these scenarios by using the pseudocode shown in Algorithm \ref{alg2}.

\begin{algorithm}[tb]
{\small
	\caption{Cycle breaking algorithm}
	\begin{algorithmic}[1]
	\STATE $result \gets new$ $List()$\\
	\IF{NEC exists}
	    \IF{NOC exists}
	        \IF{overlap of NEC and NOC exists}
	            \FOR{$node$ in $overlapped\_nodes$}
	                \STATE pick arbitrary NEC that contains $node$\\
	                \STATE break the NEC into two worlds (by Theorem \ref{theo_nec})\\
	                \STATE choose the world in which $node==True$\\
	                \STATE remove edges (by Corollary \ref{corol})\\
	                \STATE add the chosen world to $result$\\
	            \ENDFOR
	        \ELSE
	            \RETURN Unsatisfiable\\
	        \ENDIF
	   \ELSE
	        \STATE pick arbitrary NEC\\
	        \STATE break the NEC into two worlds (by Theorem \ref{theo_nec})\\
	        \STATE remove edges (by Corollary \ref{corol})\\
	        \STATE add both worlds to $result$\\
	   \ENDIF
	\ELSIF{NOC exists}
	   \RETURN Unsatisfiable\\
	\ELSE
	    \STATE assign False value to every node\\
	    \STATE remove edges (by Corollary \ref{corol})\\
        \STATE add the world to $result$
	\ENDIF
	\RETURN $result$
	\end{algorithmic}
	\label{alg2}}
\end{algorithm}

The most important thing for cycle breaking is that we need to follow a specific order with respect to types of cycles.
As we discussed in Theorem \ref{theo_nec}, NECs create worlds, NOCs kill worlds. Which means that a negative even cycle can divide a world into two, while a negative odd cycle will make one or more worlds unsatisfiable and thus disappear (this happens if no node in the NOC has the value \textit{True}). Therefore, when breaking a virtual node that has hybrid cycles, we need to first check those NOCs to make sure that these cycles are satisfiable. According to Theorem \ref{theo_noc}, a NOC will make the world satisfiable if and only if it has \textit{True} nodes. In a virtual node, the only way for a NOC to have \textit{True} node would be by overlapping with a NEC. The overlapping node is assigned \textit{True} which means that the NEC admits only possible world, that is the one with the overlapped node as \textit{True} (see Example \ref{eg1}). Refer to Algorithm \ref{alg2}, line 4-11 deals with this situation.

\begin{figure}[tb]
    \centering
    \includegraphics[scale=0.25]{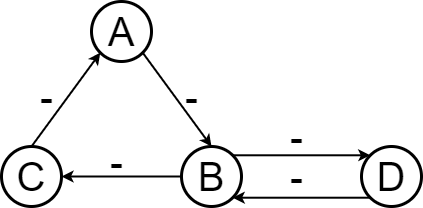}
    \caption{A Cycle Breaking Example}
    \label{fig:fig5}
\end{figure}

\begin{example}[Breaking of Overlapped Cycles]
In Figure \ref{fig:fig5}, \{$A$, $B$, $C$\} is an negative odd cycle, and \{$B$, $D$\} is a negative even cycle. We first need to make sure the negative odd cycle can be broken (i.e., a satisfying assignment made to the nodes in the NOC). In traversal, we found that node $B$ can be assigned as \textit{true} by breaking the negative even cycle \{$B$, $D$\}. Therefore, we can get a model \{A, B\} by keeping $B$ as \textit{true}.
\label{eg1}
\end{example}

When there is a NOC which does not overlaps with any NEC, no assignment is possible (line 13 and 22 in the Algorithm \ref{alg}). If there is neither NEC, nor NOC, the only possible situation is that all nodes that are connected by positive edges. Since we are in the virtual node which has no predecessors, per Theorem \ref{theo_pc}, there is no way to make these nodes \textit{True}; so we assign \textit{False} to every node, and delete all edges.

\subsection{Implementation and Performance} \label{sec:implementation} 

The \textbf{grASP} system has been written in python. A C++ version under the same system architecture design (Figure \ref{fig:fig6}) is under development. We adopted Johnson's algorithm which was proposed for finding all the elementary circuits of a directed graph ~\cite{johnson1975finding}. Our Python implementation uses the $DiGraph$ data structure and \textit{simple\_cycles} function from NetworkX ~\cite{SciPyProceedings_11}.

The performance testing for \textbf{grASP} was done on two types of programs: (i) established benchmarks such as N-queens; (ii) randomly generated answer set programs. Clingo \cite{DBLP:journals/corr/GebserKKS14} was chosen as the system to compare with. For the first phase, we chose four classic NP problems (map coloring problem, Hamiltonian cycle problem, etc.). The results are shown in Table \ref{tb1}. For the second phase, we used a novel propositional ASP program generator that we have developed for this purpose to generate random programs. The testing performed five rounds with 100 programs each (Table \ref{tb2}). The performance comparison shows that for programs with simpler cycle conditions, \textbf{grASP} achieved similar speed to Clingo, but when solving programs with large number of cycles, \textbf{grASP} is slowed down by the cycle breaking process. 

\begin{table}
\centering
\begin{tabular}{lrr}
\toprule
Problem  & Clingo & grASP \\
\midrule
Coloring-10 nodes & 0.004  & 0.693\\
Coloring-4 nodes & 0.001  & 0.068\\
Ham Cycle-4 nodes(no cross edges) & 0.001  & 0.052\\
Ham Cycle-4 nodes(fully connected) & 0.002  & 0.089\\
Birds    & 0.001  & 0.001\\
Stream Reasoning   & 0.001  & 0.001\\
\bottomrule
\end{tabular}
\caption{Performance Comparison on Classic Problems}
\label{tb1}
\end{table}

\begin{table}
\centering
\begin{tabular}{lrrrrr}
\toprule
Round & \#Rules & \#NEC & \#NOC  & Clingo & grASP \\
\midrule
1 & 3231 & 61813 & 61781 & 0.032  & 0.351\\
2 & 3078 & 28593 & 29040 & 0.027  & 0.276\\
3 & 3307 & 39346 & 39433 & 0.028  & 0.087\\
4 & 3069 & 14581 & 14868 & 0.022  & 0.405\\
5 & 3074 & 23017 & 22984 & 0.024  & 0.801\\
\bottomrule
\end{tabular}
\caption{Performance  on Random Problems (time in seconds)}
\label{tb2}
\end{table}

\begin{figure}[tb]
    \centering
    \includegraphics[scale=0.3]{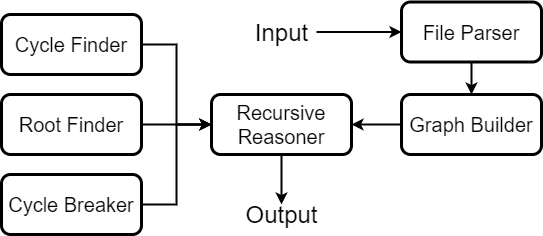}
    \caption{grASP System Architecture}
    \label{fig:fig6}
\end{figure}

\section{Causal Justification} \label{sec:causaljustification}

A major advantage of grASP is that it provides justification as to why a literal is in an answer set for free. Providing justification is a major problem for implementations of ASP that are based on SAT solvers.  In contrast to SAT-based ASP solvers, \textbf{grASP} maintains the information about structure of an ASP program while computing stable models. Indeed, the resulting graph itself is a justification tree. Since the truth values of all vertices are propagated along edges, we are able to find a justification by looking at the effective out-going edges and their ending nodes. Here the effective out-going edge refer to an edge that actually propagated \textit{True} value to its ending node. According to propagation rules that are discussed in Section \ref{sec:propagate}, there are only two type of \textit{effective} out-going edges: (i) positive edge coming from a \textit{True} node; (ii) negative edge coming from a \textit{False} node. Every effective out-going edge should point to a \textit{True} node.

\begin{example}[Graph Coloring Problem]
 Given a planar graph, color each node (red/green/blue), so that no two connected nodes have the same color.
\label{eg2}
\end{example}

Let's take  graph coloring as an example. The problem is defined in Example \ref{eg2}. For simplicity, we use a configuration with 4 nodes and 4 edges. We got 18 different answer sets. Now we want to justify one of them, which is: \textit{blue(1), red(2), blue(3), green(4)}.
The justification first picks effective out-going edges, then check each edge's ending node. If all those ending nodes are \textit{True}, the answer set is justified. Figure \ref{fig:fig7} shows a part of the justification graph. The red-circled nodes are \textit{True} nodes, while the black ones are \textit{False}. The path justifying a literal can be traced on this graph.

\begin{figure}[tb]
    \centering
    \includegraphics[scale=0.35]{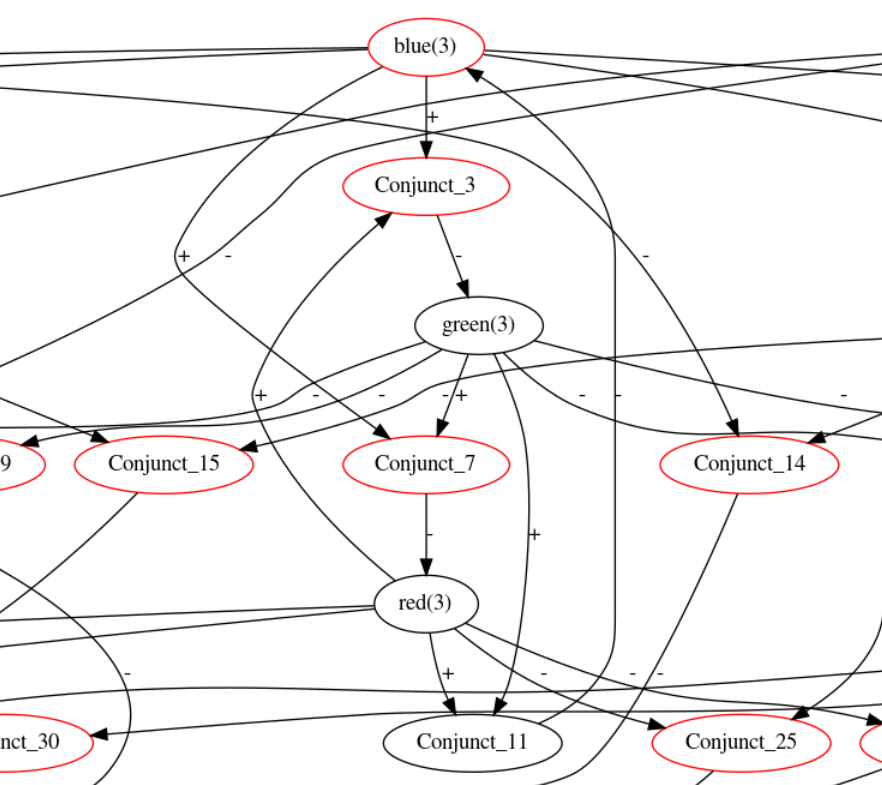}
    \caption{Part of the Justification Graph}
    \label{fig:fig7}
\end{figure}

\section{Conclusion and Future Work} \label{sec:conclusion}

We proposed a dependency graph based approach to compute the answer sets of an answer set program. We use a novel transformation to ensure that each program has a unique dependency graph, as otherwise multiple programs can have the same dependency graph. A major advantage of our algorithm is that it can produce a justification for any proposition that is entailed by a program. 
Currently, \textbf{grASP} only works for propositional answer set programs. Our goal is to extend it so that answer sets of datalog programs (i.e., answer set programs with predicates whose arguments are limited to variables and constants) can also be computed without having to ground them first. This will be achieved by dynamically propagating bindings along the edges connecting the nodes in our algorithm's propagation phase. 

The \textbf{grASP} system is being developed with three main applications in mind (i) \textbf{justification}: being able to justify each literal in the answer set, (ii) \textbf{program debugging}: if a program has no answer sets, then find (small) changes to the program (with some guidance from the user) that will make the program succeed, (iii) \textbf{commonsense reasoning}: given a query, we want to not only find out a justification for it, but also the related associative knowledge. E.g., if we infer that Tweety flies because Tweety is a bird due to the rule \textit{flies(X) :- bird(X)}., then we also want to know the associative concept that Tweety has wings from the rule \textit{haswings(X) :- bird(X)}. Additionally, \textbf{grASP} paves the way for implementing ASP without grounding, a major challenge even today notwithstanding work done by \cite{arias2018constraint,sasp}.

Even though the speed of execution on \textbf{grASP} is slower compared to CLINGO, it still finds solutions to NP-hard problems in a reasonable time. We plan to investigate optimizing techniques such as conflict driven clause learning \cite{silva2003grasp,gebser22} to speed up execution. \textbf{grASP} is more than just an ASP solver: its visualization feature makes it suitable for educational purpose and for debugging. Moreover, a graph-based approach brings new possibilities for applying optimization. All source code \& test data of this project will be made open to the public.
\bibliographystyle{named}
\bibliography{ijcai21}

\begin{thebibliography}{}

\bibitem[\protect\citeauthoryear{Anger \bgroup \em et al.\egroup
  }{2001}]{anger2001nomore}
Christian Anger, Kathrin Konczak, and Thomas Linke.
\newblock Nomore: A system for non-monotonie reasoning under answer set
  semantics.
\newblock {\em PSC 802 BOX 14 FPO 09499-0014}, page 406, 2001.

\bibitem[\protect\citeauthoryear{Arias \bgroup \em et al.\egroup
  }{2018}]{arias2018constraint}
Joaqu{\'i}n Arias, Manuel Carro, Elmer Salazar, Kyle Marple, and Gopal Gupta.
\newblock Constraint answer set programming without grounding.
\newblock {\em Theory and Practice of Logic Programming}, 18(3-4):337--354,
  2018.

\bibitem[\protect\citeauthoryear{Baral}{2003}]{baral}
C.~Baral.
\newblock {\em Knowledge representation, reasoning and declarative problem
  solving}.
\newblock Cambridge University Press, 2003.

\bibitem[\protect\citeauthoryear{Eiter \bgroup \em et al.\egroup
  }{2000}]{EFPL12}
Thomas Eiter, Wolfgang Faber, Nicola Leone, and Gerald Pfeifer.
\newblock Declarative problem-solving using the dlv system.
\newblock In {\em Logic-based artificial intelligence}, pages 79--103.
  Springer, 2000.

\bibitem[\protect\citeauthoryear{Gebser \bgroup \em et al.\egroup
  }{2007}]{gebser22}
Martin Gebser, Benjamin Kaufmann, Andr{\'e} Neumann, and Torsten Schaub.
\newblock Conflict-driven answer set solving.
\newblock In {\em IJCAI}, volume~7, pages 386--392, 2007.

\bibitem[\protect\citeauthoryear{Gebser \bgroup \em et al.\egroup
  }{2014}]{DBLP:journals/corr/GebserKKS14}
Martin Gebser, Roland Kaminski, Benjamin Kaufmann, and Torsten Schaub.
\newblock Clingo = {ASP} + control: Preliminary report.
\newblock {\em CoRR}, abs/1405.3694, 2014.

\bibitem[\protect\citeauthoryear{Gelfond and Kahl}{2014}]{gelfond2014knowledge}
Michael Gelfond and Yulia Kahl.
\newblock {\em Knowledge representation, reasoning, and the design of
  intelligent agents: The answer-set programming approach}.
\newblock Cambridge University Press, 2014.

\bibitem[\protect\citeauthoryear{Gelfond and Lifschitz}{1988}]{GL2}
Michael Gelfond and Vladimir Lifschitz.
\newblock The stable model semantics for logic programming.
\newblock In {\em ICLP/SLP}, volume~88, pages 1070--1080, 1988.

\bibitem[\protect\citeauthoryear{Hagberg \bgroup \em et al.\egroup
  }{2008}]{SciPyProceedings_11}
Aric~A. Hagberg, Daniel~A. Schult, and Pieter~J. Swart.
\newblock Exploring network structure, dynamics, and function using networkx.
\newblock In Ga\"el Varoquaux, Travis Vaught, and Jarrod Millman, editors, {\em
  Proceedings of the 7th Python in Science Conference}, pages 11 -- 15,
  Pasadena, CA USA, 2008.

\bibitem[\protect\citeauthoryear{Johnson}{1975}]{johnson1975finding}
Donald~B Johnson.
\newblock Finding all the elementary circuits of a directed graph.
\newblock {\em SIAM Journal on Computing}, 4(1):77--84, 1975.

\bibitem[\protect\citeauthoryear{Konczak \bgroup \em et al.\egroup
  }{2005}]{konczak2005graphs}
Kathrin Konczak, Thomas Linke, and Torsten Schaub.
\newblock Graphs and colorings for answer set programming.
\newblock {\em arXiv preprint cs/0502082}, 2005.

\bibitem[\protect\citeauthoryear{Lifschitz and
  Turner}{1994}]{lifschitz1994splitting}
Vladimir Lifschitz and Hudson Turner.
\newblock Splitting a logic program.
\newblock In {\em ICLP}, volume~94, pages 23--37, 1994.

\bibitem[\protect\citeauthoryear{Linke and Sarsakov}{2005}]{linke2005suitable}
Thomas Linke and Vladimir Sarsakov.
\newblock Suitable graphs for answer set programming.
\newblock In {\em International Conference on Logic for Programming Artificial
  Intelligence and Reasoning}, pages 154--168. Springer, 2005.

\bibitem[\protect\citeauthoryear{Marek and Truszczy{\'n}ski}{1999}]{MT5}
Victor~W Marek and Miroslaw Truszczy{\'n}ski.
\newblock Stable models and an alternative logic programming paradigm.
\newblock In {\em The Logic Programming Paradigm}, pages 375--398. Springer,
  1999.

\bibitem[\protect\citeauthoryear{Marple \bgroup \em et al.\egroup
  }{2017}]{sasp}
Kyle Marple, Elmer Salazar, and Gopal Gupta.
\newblock Computing stable models of normal logic programs without grounding.
\newblock {\em arXiv preprint arXiv:1709.00501}, 2017.

\bibitem[\protect\citeauthoryear{Silva and Sakallah}{2003}]{silva2003grasp}
Jo{\~a}o P~Marques Silva and Karem~A Sakallah.
\newblock Grasp—a new search algorithm for satisfiability.
\newblock In {\em The Best of ICCAD}, pages 73--89. Springer, 2003.

\bibitem[\protect\citeauthoryear{Simons \bgroup \em et al.\egroup
  }{2002}]{SNS18}
Patrik Simons, Ilkka Niemel{\"a}, and Timo Soininen.
\newblock Extending and implementing the stable model semantics.
\newblock {\em Artificial Intelligence}, 138(1-2):181--234, 2002.

\end{thebibliography}

\end{document}